\documentclass[letterpaper, 10 pt, conference]{ieeeconf}
\IEEEoverridecommandlockouts

\usepackage{amsmath}

\usepackage{mathrsfs}
\usepackage{dsfont}

\usepackage{cite}

\usepackage{amsthm,amssymb}
\usepackage{float}
\usepackage{textcomp}

\usepackage{algorithm}
\usepackage{algpseudocode}
\algnewcommand\algorithmicinput{\textbf{Input:}}
\algnewcommand\Input{\item[\algorithmicinput]}
\algnewcommand\algorithmicoutput{\textbf{Output:}}
\algnewcommand\Output{\item[\algorithmicoutput]}

\usepackage{hyperref}
\usepackage{xcolor}
\hypersetup{
    colorlinks,
    linkcolor={black},
    citecolor={blue!50!black},
    urlcolor={blue!80!black}
}
\usepackage[normalem]{ulem}
\theoremstyle{plain}
\usepackage{cleveref}
\usepackage{multirow}

\crefname{theorem}{Theorem}{Theorem}
\crefname{corollary}{Corollary}{Corollary}
\crefname{section}{Sec.}{Sec.}
\crefname{assum}{Assumption}{Assumption}
\crefname{algorithm}{Alg.}{Alg.}
\crefname{lemma}{Lem.}{Lem.}
\crefname{figure}{Fig.}{Fig.}
\crefname{table}{Table}{Table}
\crefname{definition}{Def.}{Def.}
\crefname{proposition}{Prop.}{Prop.}

\newtheorem{theorem}{Theorem}

\newtheorem{definition}{Definition}

\newtheorem{assum}{Assumption}

\newtheorem{corollary}{Corollary}
\newtheorem{proposition}{Proposition}
\renewenvironment{proof}{\noindent\textit{Proof:}}{\hfill$\blacksquare$}

\usepackage{makecell}

\usepackage{graphicx}
\usepackage{tikz}
\usepackage{float}
\usetikzlibrary{arrows}
\usetikzlibrary{arrows.meta}
\usetikzlibrary{arrows,positioning} 
\usetikzlibrary{shapes}
\usetikzlibrary{calc}
\usetikzlibrary{fadings}
\usetikzlibrary{decorations.pathreplacing}
\usetikzlibrary{decorations.text}

\usepackage{subfigure}
\makeatletter
\let\labelindent\@undefined
\makeatother
\usepackage{enumitem}
\usepackage{amsfonts}

\usepackage{setspace} 

\usepackage{scalerel}
\def\msquare{\mathord{\scalerel*{\Box}{gt}}}
\def\mdiamond{\mathord{\scalerel*{\Diamond}{gt}}}
\makeatletter \providecommand\dotdiamond{\mathpalette\@barred\mdiamond} \def\@barred#1#2{\ooalign{\hfil$#1\cdot$\hfil\cr\hfil$#1#2$\hfil\cr}}  \makeatother
\makeatletter \providecommand\dotbox{\mathpalette\@burrow\msquare} \def\@burrow#1#2{\ooalign{\hfil$#1\cdot$\hfil\cr\hfil$#1#2$\hfil\cr}}  \makeatother

\usepackage[top=54pt, left=54pt, right=54pt, bottom=54pt]{geometry}
\newcommand{\real}{\mathbb{R}}
\newcommand{\state}{\mathbf{x}}
\newcommand{\pose}{\mathbf{p}}
\newcommand{\statedim}{{n_\state}}
\newcommand{\stateset}{\mathcal{X}}
\newcommand{\ctrl}{\mathbf{u}}
\newcommand{\ctrldim}{{n_\ctrl}}
\newcommand{\ctrlset}{\mathcal{U}}
\newcommand{\disturb}{\mathbf{d}}
\newcommand{\disturbdim}{{n_\disturb}}
\newcommand{\disturbset}{\mathcal{D}}

\newcommand{\local}{^{\rm loc}}
\newcommand{\glob}{^{\rm G}}
\newcommand{\trueVal}{^{\dagger}}

\newcommand{\ham}{\mathcal{H}}
\newcommand{\safe}{\mathcal{L}}
\newcommand{\unsafe}{\mathcal{G}}
\newcommand{\reach}{\mathcal{R}}
\newcommand{\obs}{\mathcal{O}}

\newcommand{\failrate}{\epsilon}
\newcommand{\confidence}{\beta}
\newcommand{\desired}{_{\rm des}}

\newcommand{\floor}{_{\rm floor}}
\newcommand{\calib}{_{\rm cal}}

\newcommand{\robustDelta}{\delta'}
\newcommand{\cp}{_{\rm CP}}

\newcommand{\hyparam}{\phi}
\newcommand{\hyparamSet}{\Phi}
\newcommand{\operator}{\Psi}
\newcommand{\operatorTrue}{\operator\trueVal}
\newcommand{\learned}{_\hyparam}
\newcommand{\pdeInput}{\mathcal{A}}
\newcommand{\pdeOutput}{\mathcal{Y}}

\newcommand{\nominal}{_{\rm n}}
\newcommand{\perturb}{_{\rm p}}
\newcommand{\induced}{_{\phi}^\mathrm{n}}

\usepackage{booktabs}
\usepackage{makecell}
\usepackage{pifont}

\usepackage{caption}
\renewcommand{\thetable}{\Roman{table}}

\begin{document}

\title{\textbf{\huge Contingency-Aware Planning via Certified Neural Hamilton–Jacobi Reachability}\vspace*{3pt}}

\author{Kasidit~Muenprasitivej and Derya~Aksaray
\thanks{The authors are with the Department of Electrical and Computer Engineering, Northeastern University, Boston, MA, USA, 02115 (\{muenprasitivej.k, d.aksaray\}@northeastern.edu)}
}

\maketitle

\thispagestyle{empty}

\begin{abstract}
Hamilton-Jacobi (HJ) reachability provides formal safety guarantees for dynamical systems, but repeatedly solving the HJ partial differential equation in changing environments limits its use in real-time planning. This paper introduces learning-based reachability and an online safety certification that together enable real-time contingency-aware planning in unknown environments with probabilistic safety guarantees. We use Fourier Neural Operator (FNO) to approximate the HJ solution operator under varying obstacle configurations. To certify the learned reach-avoid set, we recast robust scenario optimization as multiple hypothesis testings over its sublevel sets, which yields a \textit{simultaneous} guarantee across candidate levels and recovers the largest certifiable volume at the prescribed safety level. Within this certified set, the induced recovery policy is guaranteed to return the system to a safe region in \textit{finite time} under any admissible disturbance. Numerical results show that the learning and certification pipeline updates reach-avoid sets at real-time rates and is demonstrated on youBot simulation\footnote{youBot simulation video: \url{https://youtu.be/8KDjjK2ydgs}\label{fn:youtube}}.
\end{abstract}
\section{Introduction}
Ensuring provably safe autonomy while maintaining computational tractability remains a central challenge in robotic planning. Contingency planning resolves this by equipping the system with certified fallback maneuvers that guarantee recovery under adversarial events. In this work, we propose a learning-enabled Hamilton--Jacobi (HJ) reach-avoid framework that guarantees a finite-time recovery path to a safe region throughout navigation in unknown environments.

\noindent \textbf{Contingency Planning and Safety Certificates:}
Contingency planning frameworks commonly employ backup behaviors paired with safety certificates. For example,  
Model Predictive Control (e.g., \cite{kerrigan2018:mpc}) enforces safety through online constrained optimization, but may suffer infeasibility under tight computational budgets. Control Barrier Functions (e.g., \cite{ames2017:cbf}) cast safety enforcement as a quadratic program, but require the construction
of valid barrier certificates. Precomputed fallback policies (e.g., \cite{chen2021:backupcbf}) switch to predetermined safe behaviors upon contingency detection, but require careful design of the switching logic.
HJ reachability \cite{margellos2011:hjReachAvoid} provides certified reach-avoid set with the optimal recovery policy given by spatial gradient of the set, but solving the underlying Hamilton–Jacobi–Isaacs variational inequality (HJI-VI) is costly under changing environments. Learning-based models can provide such efficiency at the cost of safety guarantees.

\noindent  \textbf{Learning-Based Hamilton--Jacobi Reachability:} 
Recent works approximate the HJI-VI solution using neural PDE solvers trained via residual minimization (e.g., DeepReach~\cite{tomlin2021:deepReach}, ExactBC~\cite{singh2025:exactbc}, and physics-informed approaches~\cite{tayal2025:PIML-HJ}). These methods learn a fixed PDE instance and require retraining when dynamics, target sets, or horizons change, which limits their use in initially unknown environments. Parametric extensions to DeepReach~\cite{Borquez2023:parameterHJ} learn a family of reach-avoid problems (e.g., varied target set locations) but remain confined to low-dimensional parameter families. Alternative approaches reformulate reach-avoid problems via discrete-time Bellman operators~\cite{hsu2021:reachavoidrl} or classification-based policy learning~\cite{rubies2019:classifcation-based-hj}. Although computationally attractive, these methods do not enforce full HJI-VI structural consistency, and safety certification relies on supervisory mechanisms rather than intrinsic PDE guarantees. Most closely related is HJRNO~\cite{li2025:hjrno}, which employs Fourier Neural Operators to learn an HJ solution operator that generalizes across obstacle configurations. However, it targets infinite-horizon invariant sets rather than finite-time reachability and provides no formal safety certification of the learned safe sets.

\noindent \textbf{Certifying Neural HJ Reachable Set:}
Several works provide methods for certifying learned reachable sets.
Counterexample-guided synthesis~\cite{yang2025:scalableSynthesisHJ} verifies constraint satisfaction of the learned value function, but the certificate only holds for the piecewise-linear neural surrogate for nonlinear systems.
Satisfiability Modulo Theories certifies a uniform bound on the learned HJ residual which gives sound over- or under-approximation of the true reachable set~\cite{solanki2025:CARe, solanki2026:certifyHJRL}, but the reachable set--induced control policy remains uncertified.
To ensure rollout safety, sampling-based reachability~\cite{lew2022:randup} over-approximates the true forward reachable set for any black-box system under a given policy.
Scenario optimization~\cite{dietrich2025:holdout} provides probabilistic guarantees that entire rollout trajectories remain within an estimated reachable tube.
Yet these methods verify a prescribed initial set rather than characterize the set of all safe initial states.
Scenario-based extensions in \cite{lin2023:sceneOpt, Lin2024:robustSceneOpt} address this by certifying the safe initial states encoded by the learned value function, which guarantees the inclusion in the true reachable set and the safety of the learned policy.
However, their guarantee holds only for a single sublevel set, which is often conservative and requires additional convex programs to refine~\cite{smith2026:hierarchical}.
We instead pose certification as a multiple hypothesis test, which yields a \textit{simultaneous}
guarantee over sublevel sets and recovers the largest certifiable set at the desired safety level.


\noindent \textbf{Contributions:}
This paper presents a contingency-aware planning framework with probabilistic finite-time reachability guarantees. Our main contributions are: (1) we learn the \emph{finite-time} 
HJI-VI solution operator with a Fourier Neural Operator (FNO) and provide probabilistic guarantees to the learned reach-avoid sets under worst-case disturbances; (2) we recast robust scenario optimization as a multiple hypothesis test that certifies a family of sublevel sets \textit{simultaneously}. This permits a search for the largest provably safe reach-avoid set with a certified fallback; (3) we propose an online planning framework that uses the learned FNO to rapidly update reach-avoid sets as obstacles are revealed. The resulting framework enables safe navigation in initially unknown environments with probabilistic finite-time reachability guarantees to a designated safe region.


\section{Preliminaries}\label{sec:prelim}
\subsection{Hamilton-Jacobi Reachability}\label{sec:hjr}
We follow the standard HJ reachability formulation \cite{claire2017:hjrOverview}. The system state $\state \in \stateset \subset \real^\statedim$ evolves as $\dot{\state} = f(\state,\ctrl,\disturb)$, with control $\ctrl \in \ctrlset \subset \real^\ctrldim$ and disturbance $\disturb \in \disturbset \subset \real^\disturbdim$. 
Let $\xi^{\ctrl,\disturb}_{\state,t}(\tau)$ denote the state reached at time $\tau\geq t$ starting from $(\state,t)$ under control $\ctrl(\cdot)$ and disturbance $\disturb(\cdot)$. 
Given a target set $\safe\subset\stateset$ and unsafe set $\unsafe\subset\stateset$, the backward reach-avoid tube (BRAT) at time $t\in[0,T]$ is

{\small
\begin{equation*}
\begin{aligned}
  \reach(t)= \big\{ \state \;\big|\;
    &\exists \ctrl(\cdot),\ \forall \disturb(\cdot),\ \exists\tau\in[t,T]:\
    \xi^{\ctrl,\disturb}_{\state,t}(\tau) \in \safe \\
    &\wedge\ \xi^{\ctrl,\disturb}_{\state,t}(s) \notin \unsafe
    \ \ \forall s \in [t,\tau] \big\},
\end{aligned}
\end{equation*}
}

\noindent i.e., set of all states, starting at time $t$, from which the system can reach $\safe$ \textit{by the deadline} $T$ and avoid $\unsafe$ until then.

Define $\safe = \{\state \mid \ell(\state) \le 0\}$ and $\unsafe = \{\state \mid g(\state) \ge 0\}$ for continuous functions $\ell, g : \real^\statedim \to \real$. The reach-avoid condition is encoded in cost functional
\begin{equation}\label{eq:hjr-opt-cost}
    \small J(\state,t,\ctrl(\cdot),\disturb(\cdot)) =
    \min_{\tau \in [t,T]} \max\Big\{
      \ell\big(\xi^{\ctrl,\disturb}_{\state,t}(\tau)\big),\,
      \max_{s \in [t,\tau]} g\big(\xi^{\ctrl,\disturb}_{\state,t}(s)\big)
    \Big\}
\end{equation}
where $J < 0$ is the feasibility condition.
This yields a zero-sum differential game, and the value function
\begin{align}\label{eq:value-function}\small
V(\state,t) := \inf_{\ctrl(\cdot)}\sup_{\disturb(\cdot)}\, J\!\left(\state,t,\ctrl(\cdot),\disturb(\cdot)\right)
\end{align}
is the unique viscosity solution (e.g., 
\cite{margellos2011:hjReachAvoid}) 
of the HJI-VI
\begin{equation}\label{eq:hji-vi}
\small
\begin{aligned}
\max\Big\{ &\min\big\{ \partial_t V(\state,t) + \ham\big(\state,\nabla_\state V(\state,t)\big),
\ell(\state) - V(\state,t) \big\},\;\\[-2pt]
&g(\state) - V(\state,t) \Big\} = 0,
\end{aligned}
\end{equation}
with terminal condition $V(\state,T) = \max\{\ell(\state), g(\state)\}$, and the hamiltonian is defined as
$
\ham(\state,t) = \sup_{\disturb\in\mathcal{D}}\inf_{\ctrl\in\mathcal{U}}\left\langle \nabla_x V(\state,t),\, f(\state,\ctrl,\disturb)\right\rangle.
$

The BRAT is thus the zero sublevel set $\reach(t) = \{\state \mid V(\state,t) \le 0\}$, and the optimal feedback control enforcing reachability under worst-case disturbances is
\begin{equation}\label{eq:hjr-opt-ctrl}
\ctrl^\star(\state,t) = \arg\min_{\ctrl\in\mathcal{U}}\sup_{\disturb\in\mathcal{D}}\left\langle \nabla_x V(\state,t),\, f(\state,\ctrl,\disturb)\right\rangle.
\end{equation}


\subsection{Fourier Neural Operator}\label{sec:fno}
Let $D=(\mathbf{D'} \times \mathbf{H}) \subset \real^d$ be a bounded open domain, where $\mathbf{D'}$ denotes primary domain and $\mathbf{H}$ denotes parametric domain. 
Let $\pdeInput = \pdeInput(D;\real^{d_a})$ and 
$\pdeOutput = \pdeOutput(D;\real^{d_\mathrm{y}})$ denote separable Banach spaces of 
functions on $D$. 
Given a solution operator 
$\operatorTrue:\pdeInput \to \pdeOutput$ associated with a parametric PDE, the objective is to approximate $\operatorTrue$ by a neural operator $\operator_\phi : \pdeInput \to \pdeOutput$ with parameters $\phi \in \Phi$, where $\Phi$ is a finite-dimensional parameter space and $\phi$ is learned from
observations $\{a_j, \mathrm{y}_j\}_{j=1}^N$, with $a_j$ drawn i.i.d.\ from a probability measure $\mu$ supported on $\pdeInput$ and
$\mathrm{y}_j = \operatorTrue(a_j)$.
The Fourier Neural Operator (FNO)~\cite{li2020:neuralOperator} is an architecture that learns kernels in the Fourier domain for approximating solution operators $\operatorTrue$. Its single set of parameters $\hyparam\in\hyparamSet$ generalizes across different discretizations and enables zero-shot super-resolution. Given an 
input function $a \in \pdeInput$, the learned operator $\operator\learned$ with $B\in\mathbb{N}$ Fourier layers is defined as
\begin{equation}
\nonumber
\mathrm{y} = \operator\learned(a)
:= \Pi \circ \mathrm{Layer}_{B-1} \circ \cdots \circ
\mathrm{Layer}_{0} \circ \Lambda(a),
\end{equation}
where $\Lambda:\real^{d_a}\to\real^{d_\mathrm{v}}$ and
$\Pi:\real^{d_\mathrm{v}}\to\real^{d_\mathrm{y}}$ denote the pointwise lifting
and projection operators, respectively. Define $\mathrm{v}_0 = \Lambda(a)$, and let $\mathrm{v}_b$ be a latent function taking values in $\real^{d_\mathrm{v}}$ with $b=0,\ldots,B-1$.
{\small
\begin{equation*}
\mathrm{v}_{b+1}(z) = \mathrm{Layer}_b(\mathrm{v}_b)(z) \doteq \sigma\!\left(
W \mathrm{v}_b(z) + \mathcal{F}^{-1}\!\bigl(R_\phi \cdot
\mathcal{F}(\mathrm{v}_b)\bigr)(z) \right),
\end{equation*}
}
\noindent where  $z \in D$, $W:\real^{d_\mathrm{v}}\to\real^{d_\mathrm{v}}$ is a pointwise
linear map, $\sigma$ is an activation function, $\mathcal{F}$ and
$\mathcal{F}^{-1}$ are the Fourier transform and its inverse, respectively,
and $R_\phi$ is a learned tensor that applies spectral multiplication in Fourier domain. 


\section{Problem Definition}\label{Section: Problem}
Consider a mobile robot navigating in a $2$D planar environment $\Omega \subset \real^2$, which contains a set of known safe regions $\mathbf{L}=\{\safe_i\}_{i=1}^{L}$ and initially unknown obstacles $ \mathbf{O}=\{\obs_j\}_{j=1}^H$. The obstacles are revealed online through onboard sensing. 

The robot evolves according to $\dot \state = f(\state,\ctrl,\disturb)$, where $\state(t)\doteq (\pose_{\rm robot},\theta) \in \stateset \subset\real^3$ is the robot's state  at time $t\geq 0$ with position $\pose_{\rm robot} \doteq(x,y)\in \real^2$ and orientation $\theta\in\mathbb{S}^1$,  $\ctrl(t) \in \ctrlset \subset \real^\ctrldim$ is the control input, and $\disturb(t) \in \disturbset \subset \real^\disturbdim$ is a bounded disturbance. The robot is equipped with omnidirection sensor of radius $r_{\rm sense}>0$, so the obstacle is detected once it lies within the sensing range. We now state the following geometric assumption.

\begin{assum}\label{assum:disjoinSafeUnsafe}
Each safe region $\safe_i$ is a closed disk of radius $r_{\ell,i}>0$ centered at $\pose_{\ell,i}$. Each detected obstacle $\obs_j$ is represented as a closed disk of radius $r_{\obs_j}>0$ 
centered at $\pose_{\obs_{j}}$. Both regions are disjoint: $\safe_i \cap \obs_j = \emptyset$ for all $i, j$.
\end{assum}

\textbf{Problem Statement}:
Under \cref{assum:disjoinSafeUnsafe}, consider a robot with dynamics $\dot{\state} = f(\state, \ctrl, \disturb)$ that starts at $\pose_0$ in $\Omega$ with $L$ pre-defined safe sets $\mathbf{L}$ and $H$ unknown obstacles $\mathbf{O}$ revealed online by sensing. Given a violation level $\failrate\desired \in (0,1)$, design a planning framework that generates
$(1)$ a global path from $\pose_0$ to the goal $\pose_{\mathrm{goal}}$; and
$(2)$ a reachability-based contingency policy at every state $\state$ along the global path, which safely drives the robot from $\state$ into some $\safe_i \in \mathbf{L}$ within a finite horizon, with probability at least $1-\failrate\desired$.


\section{Proposed Method}

\subsection{Operator Learning for HJI-VI}

\subsubsection{FNO in Reach-Avoid Problem}\label{sec:HJRFNO}
Inspired by \cite{li2025:hjrno}, we implement the FNO architecture of Sec.~\ref{sec:fno} to learn the solution operator $\operator\trueVal$ of HJI-VI \eqref{eq:hji-vi}.
The FNO model is trained on a local spatial domain $D\local \subset \real^2$ centered on the safe region $\safe\local$. We later show that the learned solution can be translated into the global planning frame for any predefined safe region $\safe_i \subseteq \Omega$, $i = 1, \dots, L$.
The surrogate $\operator\learned$ maps the obstacle configuration $g(\state)$ and time constant $t$ to the viscosity solution $V\trueVal(\state,t)$ over the state-time domain $D = D\local \times \mathbb{S}^1 \times [0,T]$. 
Hereafter, the superscript $\dagger$ denotes the true solution and subscript $\hyparam$ the learned approximation. 

The safe set is defined on the state set $\stateset\doteq D\local\times\mathbb{S}^1$ as 
\begin{equation} \label{eq:safeSet}
    \safe\local := \left\{ \state \doteq (\pose,\theta) \in \stateset \;\middle|\; \|\pose\|_2 \leq r_{\ell} \right\},
\end{equation}
i.e., a closed cylinder of radius $r_{\ell}$ centered at the origin.

For the $j$-th obstacle with center $\pose_{\obs_j} \in D\local$ and radius 
$r_{\obs_j} \in [r_{\obs_{\min}}, r_{\obs_{\max}}]$, $r_{\obs_{\min}} > 0$, let
$g_j(\state) = r_{\obs_j} - \|\pose - \pose_{\obs_j}\|_2$, so that $g_j \geq 0$ 
on its interior. Given $h \in \{0,\ldots ,H\}$ observed obstacles, the unsafe set is
\begin{equation}\label{eq:unsafeSet}
    \unsafe\local := \{ \state \doteq (\pose,\theta) \in \stateset \mid g(\state) \geq 0 \},
\end{equation}
where $g(\state) = \max_{0 \leq j \leq h} g_j(\state)$ encodes the union of 
obstacles via pointwise maximum. We denote by $\mathbf{G}$ the class of functions $g$ with $\unsafe\local$ satisfying \cref{assum:disjoinSafeUnsafe}.

Extending the formulation of \cite{li2025:hjrno} to explicitly capture finite-time reachability, we take a subset of the time-augmented state space, $(x,y,t) \in \mathbf{D'} \doteq D\local \times [0,T]$, as the primary domain and treat the remaining state variable as a parametric variable $\theta \in \mathbf{H} \doteq \mathbb{S}^1$. Under this representation, the learning problem is formulated as
\begin{equation*}
    a(x,y,t,\theta) = g(\state) \in \pdeInput, \qquad
    \mathrm{y}(x,y,t,\theta) = V\trueVal(\state,t) \in \pdeOutput.
\end{equation*}
Since $g$ is time-invariant, it is extended constantly over $[0,T]$, so that $\pdeInput$ and $\pdeOutput$ are both function spaces on $D = \mathbf{D'} \times \mathbf{H}$.

The training data are generated using \texttt{OptimizedDP}~\cite{bui2022:optimizeddp}
which numerically solve \eqref{eq:hji-vi} on a coarse spatial grid over discrete time. This gives
$
\mathcal{S}_{\rm train} = \big\{ \big( g_i(\cdot,\cdot,\theta_j),\ V_i\trueVal(\cdot,\cdot,\theta_j,\cdot) \big) \,:\, i = 1,\dots,N_g,\ j = 1,\dots,N_{\theta} \big\},
$
where $g_i$ is the $i$-th obstacle configuration and $V_i\trueVal$ is its exact value function. Each function is sliced at $\theta_j$ across $N_{\theta}$ headings, so a single solution pair $(g_i, V_i\trueVal)$ yields $N_\theta$ training samples and $|\mathcal{S}_{\rm train}| = N_g N_\theta$. 
The FNO model is trained by minimizing
\begin{equation}\nonumber
    \min_{\hyparam}
    \frac{1}{|\mathcal{S}_{\rm train}|}\sum_{k=1}^{|\mathcal{S}_{\rm train}|}
   (1-\lambda)\left\| e_k\right\|_{\infty}
     + \lambda
    \left\| e_k \right\|_{2},
\end{equation}
where $e_k:=V_{k}\trueVal(\cdot) - \operator\learned ( g_{k}(\cdot) )$, the $L_\infty$ term penalizes the worst-case value error, the $L_2$ term regulates the average pointwise error for training stability, and $\lambda \in[0,1]$.

\subsubsection{Design Rationale for Real-Time Queries}\label{sec:FNO-choice-support}:
Although solving the HJI-VI is tractable for $\statedim = 3$, it becomes a bottleneck in our setting, where reach--avoid sets must be frequently recomputed as new obstacles are revealed.
Existing work relies on problem-specific acceleration, such as
GPU-parallelization (e.g., \cite{srirangam2026:scramppi}) or warm-starting from the previous solution (e.g., \cite{bajcsy2019:efficientReach}). 
In contrast, each FNO forward pass scales quasi-linearly with the $(x,y,t)$-grid size~\cite{li2021:fno}, with one pass per heading slice ($N_\theta$ in total).
\cref{tab:reach-time-sweep} compares FNO's computation time against \texttt{OptimizedDP}~\cite{bui2022:optimizeddp}.

Learning-based models are therefore preferable under frequent updates but require additional treatment to recover safety guarantees which will be addressed in \cref{sec:safeTheoryForFNO}. The most closely related model is parameter-conditioned DeepReach~\cite{Borquez2023:parameterHJ}, hereafter denoted as $\mu$-DeepReach, which learns a family of reachability problems under varying system and environment conditions parameterized by a vector $\mu$. 
While this is applicable to reach-avoid problem across varying obstacle configurations, each component of $\mu$ is scalar-valued and its dimension is fixed at training time. For instance, $H$ obstacles are representable as $\mu = [x_\obs^1, y_\obs^1, r_\obs^1, \dots, x_\obs^H, y_\obs^H, r_\obs^H]^\top \in \mathbb{R}^{3H}$, but $H$ is then fixed: fewer obstacles require ad hoc padding outside the training distribution, and more obstacles require retraining. By contrast, FNO takes a function-valued input and generalizes across obstacle maps without a fixed parameterization. Further analysis is provided in \cref{sec:fno-vs-deepreach}.


\subsection{Certification of Reach-Avoid Sets}\label{sec:safeTheoryForFNO}
\subsubsection{Probabilistic Guarantees for a Single Sublevel Set}
Define the approximated solution as $V\learned(\state,t):=\operator\learned(g(\state,t))$ and let $\reach^{\delta}\learned(0)=\{\state \mid V\learned(\state,0)\leq \delta\}$
denote its $\delta$-sublevel set. The induced control policy $\ctrl\induced(\state,t)$ is obtained by evaluating \eqref{eq:hjr-opt-ctrl} with $V\learned(\state,t)$ along the \textit{disturbance-free} rollout from $\state$. Given a failure rate $\failrate\in(0,1)$ and a confidence parameter $\confidence\in(0,1)$, 
our objective is to find the \textit{least conservative} $\reach^{\delta}\learned(0)$ such that the following holds with confidence at least $1-\confidence$: 
\begin{enumerate}[label=(\roman*), ref=(\roman*)]
 \item \label{itm:inside-true-set} The state $\state\in\reach^{\delta}\learned(0)$ is certified to lie inside the \textit{true} BRAT (i.e., $V\trueVal(\state,0)<0$) with probability at least $1-\failrate$.
  \item \label{itm:safe-reach-avoid} For every admissible $\disturb\in\disturbset$, the trajectory $\xi^{\ctrl\induced,\disturb}_{\state,0}$ starting from $\state\in\reach^{\delta}\learned(0)$ under policy $\ctrl\induced(\cdot)$ satisfies $ \xi^{\ctrl\induced,\disturb}_{\state,0}(\tau) \in \safe$ for some $\tau\in[0,T]$ and $\xi^{\ctrl\induced,\disturb}_{\state,0}(s) \notin \unsafe$ for all $s\in[0,\tau]$ with probability at least $1-\failrate$.
\end{enumerate}

\begin{assum}[Bounded position deviation]\label{assum:error-is-bounded}
Let the nominal state $\state\nominal=(\pose\nominal,\theta\nominal)$ solve $\dot\state\nominal = f(\state\nominal,\ctrl\induced)$
with $\disturb\equiv 0$, and let the perturbed state $\state\perturb=(\pose\perturb, \theta\perturb)$ solve
$\dot\state\perturb = f(\state\perturb,\ctrl\induced,\disturb)$ under the same nominal control signal $\ctrl\induced(\cdot)$ and any admissible disturbance
$\disturb(.)\in\disturbset$ with $\disturbset$ compact, and with
$\state\nominal(0)=\state\perturb(0)$. Then there exists a known,
non-decreasing $r_\pose:[0,T]\to\real_{\ge 0}$ with $r_\pose(0)=0$ such
that the position error
$e_\pose(t) := \pose\perturb(t)-\pose\nominal(t)$ satisfies
\begin{equation}\label{eq:tube}
  \|e_\pose(t)\|_2 \;\le\; r_\pose(t)
  \qquad \forall t\in[0,T],\ \forall \disturb(\cdot).
\end{equation}
\end{assum}

\cref{assum:error-is-bounded} bounds the position deviation uniformly over all admissible $\disturb(\cdot)$. This bound allows us to construct the robust cost bound \eqref{eq:cost-upper-bound} which is evaluated solely on a disturbance-free trajectory rollout. We now establish the probabilistic safety guarantee in \cref{theorem:robustSceneOpt}.

\begin{theorem}[Robust Scenario-Based Probabilistic Safety Under Disturbance]
\label{theorem:robustSceneOpt}
Suppose that Assumption~\ref{assum:error-is-bounded} holds. Define the robust cost bound
\begin{equation}\label{eq:cost-upper-bound}
\small
\begin{aligned}
    \hat J(\state\nominal,t,\ctrl\nominal) =
\min_{\tau \in [t,T]} \max &\Big\{
  \ell\big(\state\nominal(\tau)\big) + r_\pose(\tau), \\
  &\max_{s \in [t,\tau]}\big[\,g\big(\state\nominal(s)\big) + r_\pose(s)\,\big]
\Big\}.
\end{aligned}
\end{equation}
Let $\{\state_i\}_{i=1}^{N}$ be $N$ i.i.d.\ samples drawn uniformly from $\reach^{\delta}\learned(0)$.
For each sample, let $\ctrl_{\mathrm{n},i}=\ctrl\induced(\cdot)$ be the induced control policy with $\disturb\equiv0$ along $\state_i$. Let
$k := |\{\,i : \hat J(\state_i,0,\ctrl_{n,i}) \ge 0\,\}|$ be the number of unsafe trajectories. Then for any
$\confidence\in(0,1)$ and $\failrate\in(0,1)$ satisfying
\begin{equation}\small
    \mathbb{P}\Big(\mathrm{Bin}(N,\failrate) \leq k\Big) = \sum_{j=0}^{k} \binom{N}{j} \failrate^{j} (1-\failrate)^{N-j}
    \leq \confidence,
    \label{eq:scenario_bound}
\end{equation}
it holds with confidence level of at least $1-\confidence$ over the uniform sampling of
$\{\state_i\}_{i=1}^{N}$ that
\begin{equation}\label{eq:scenario_conclusion}
    \mathbb{P}_{\state\in\reach^{\delta}\learned(0)}\!\left(V\trueVal(\state,0)<0\right)
    \geq 1-\failrate .
\end{equation}
\end{theorem}
\begin{proof}
By \cite{Lin2024:robustSceneOpt}, if \eqref{eq:scenario_bound} holds then $\mathbb{P}_{\state\in\reach^{\delta}\learned(0)}(\hat J(\state,0,\ctrl\nominal)<0) \ge 1-\failrate$ with confidence $1-\confidence$. It therefore suffices to
show $V\trueVal(\state,0) \le \hat J(\state,0,\ctrl\nominal)$ pointwise.

The functions $\ell$ and $g$ in \eqref{eq:safeSet}--\eqref{eq:unsafeSet} are globally $1$-Lipschitz in $\pose$ and independent of $\theta$; since $g=\max_i g_i$ inherits the Lipschitz constant of its members, this holds for any number of obstacles. Combining this property with \cref{assum:error-is-bounded} gives the one-sided bounds
$\ell(\state\perturb(t)) \le \ell(\state\nominal(t)) + r_\pose(t)$ and $g(\state\perturb(t)) \le g(\state\nominal(t)) + r_\pose(t)$, valid for every $t\in[0,T]$ and every admissible $\disturb(\cdot)$.
Thus, for a fixed $\disturb(\cdot)$, the monotonicity of $\max$ imply that the inner
expression of $J(\state,0,\ctrl\nominal,\disturb)$ in \eqref{eq:hjr-opt-cost} is bounded above by
that of $\hat J(\state,0,\ctrl\nominal)$, and taking the minimum over
$\tau$ preserves this inequality. 
Since $\hat J(\state,0,\ctrl\nominal)$
does not depend on $\disturb(\cdot)$, we obtain
$\sup_{\disturb(\cdot)} J(\state,0,\ctrl\nominal,\disturb) \le
\hat J(\state,0,\ctrl\nominal)$. Finally, $\ctrl\nominal$ is admissible, so
by \eqref{eq:value-function} $V\trueVal(\state,0) \le
\sup_{\disturb(\cdot)} J(\state,0,\ctrl\nominal,\disturb)$, which
establishes \eqref{eq:scenario_conclusion}.
\end{proof}

By \cref{theorem:robustSceneOpt}, claim \ref{itm:safe-reach-avoid} is
constructive, in that the \textit{nominal} control policy $\ctrl\induced(\cdot)$ induced by $V\learned(\state,t)$ is probabilistically guaranteed to succeed against every admissible $\disturb(\cdot)$ given by $\sup_{\disturb(\cdot)} J(\cdot) \le
\hat J(\cdot)<0$; the membership in $\reach\trueVal(0)\doteq\{\state\mid V\trueVal(\state,0)\leq 0\}$
asserted by \ref{itm:inside-true-set} is a consequence of this. We note that in the disturbance-free case, $r_\pose \equiv 0$, so
$\hat J(\cdot) = J(\cdot)$ and \cref{theorem:robustSceneOpt} reduces to the robust scenario-based guarantee of \cite{Lin2024:robustSceneOpt}. Finally, $\confidence$ is chosen to be small (e.g., $10^{-6}$) such that it does not have any practical importance.

\begin{corollary}[Deviation bound for the unicycle]\label{cor:unicycle-tube}
In this work, we consider the unicycle with additive disturbance, with state
$\state=(x,y,\theta)\in\stateset$, control $\ctrl=(v,\omega)\in\ctrlset$, and
disturbance $\disturb=(d_x,d_y,d_\theta)\in\disturbset$:
\begin{equation}\label{eq:unicycleModel}
  \dot{x}=v\cos\theta + d_x,\;\;
  \dot{y}=v\sin\theta + d_y,\;\;
  \dot{\theta}=\omega + d_\theta,
\end{equation}
where $v\in[v_{\min},v_{\max}]$, 
$\omega\in\omega_{\max}[-1,1]$, $\|(d_x,d_y)\|_2\le\bar d$, and
$|d_\theta|\le\bar d_\theta$. Then \cref{assum:error-is-bounded} holds with
\begin{equation}\label{eq:unicycle-rp}
  r_\pose(t) = \bar d\, t + \tfrac{1}{2} v_{\max}\bar d_\theta\, t^2 .
\end{equation}
\end{corollary}

\begin{proof}
Apply the same control $\ctrl\induced(\cdot)$ to the nominal ($\disturb\equiv 0$) and
perturbed systems, and let $e_\theta := \theta\perturb-\theta\nominal$ and
$e_\pose := \pose\perturb-\pose\nominal$, with $e_\theta(0)=0$ and
$e_\pose(0)=0$. Since $\dot e_\theta = d_\theta$, we have
$|e_\theta(t)|\le\bar d_\theta t$. By the chord identity and
$|\sin z|\le|z|$,
\begin{equation}\nonumber\small
  \|\dot e_\pose\|_2 
  \boldsymbol{\le} 2 v_{\max}\big|\sin(e_\theta/2)\big| + \bar d
  \boldsymbol{\le} v_{\max}|e_\theta| + \bar d 
  \boldsymbol{\le} v_{\max}\bar d_\theta\, t + \bar d .
\end{equation}
Integrating gives $\|e_\pose(t)\|_2\le r_\pose(t)$ in \eqref{eq:unicycle-rp},
which is non-decreasing on $[0,T]$ with $r_\pose(0)=0$.
\end{proof}

\subsubsection{Hypothesis Testing over Multiple Level Sets}\label{sec:set-certification}

We remark that \cref{theorem:robustSceneOpt} certifies a single pre-defined $\delta$-sublevel set. On an unseen obstacle map, a fixed $\delta$ may not always satisfy \eqref{eq:scenario_bound} at the desired level $(\failrate\desired, \confidence\desired)$. We therefore adopt a two-stage scheme, adapted from~\cite{Lin2024:robustSceneOpt} and~\cite{angelopoulos2025:learnthentest}, that separates a configuration-independent floor from a configuration-specific certificate. \textit{(a) Offline risk control} certifies a conservative floor $\delta\floor$ over random rollouts across random maps drawn from $\mathbf{G}$. $\{V\learned \le \delta\floor\}$ is a certified set at stricter rate $\failrate\calib < \failrate\desired$ with confidence $1-\confidence\calib$ over the calibration set sampling. 
Because $\delta_{\rm floor}$ is chosen as the 
sublevel threshold that certifies against all of the sampled obstacle configurations and initial states, it may be conservative for a particular queried configuration and discard admissible volume.
\textit{(b) Online volume recovery} thus re-certifies against the queried reach-avoid set and searches for the \textit{largest} sublevel set of $V\learned$ that satisfies \eqref{eq:scenario_bound} at rate $\failrate\desired$ with confidence $1-\confidence\desired$. The two guarantees rest on disjoint samples, so $\delta\floor$ remains a valid fallback when there is no larger sublevel set that certifies online. In \cref{sec:design-constraint-for-contingency}, we consider the $\delta\floor$-sublevel sets of multiple BRATs and enforce sufficient overlap between adjacent sets to ensure global plan feasibility to the goal.

\paragraph{Offline Risk Control}
Given $(\failrate\calib, \confidence\calib)$, we seek to find the \textit{shallowest} $\delta\floor$ whose sublevel set can be certified against the FNO's prediction error on an average rollout across obstacle configurations in $\mathbf{G}$. For this, we adopt \textit{Learn-Then-Test} (LTT)\cite{angelopoulos2025:learnthentest}, a distribution-free calibration framework that recasts risk control as $P$ hypothesis testings. It returns a set of parameter values (e.g., $\delta\in\Delta$) that \emph{simultaneously} control the risk at given level $\failrate\calib$, with confidence at least $1-\confidence\calib$ over the calibration set. 

\begin{definition}[Risk-controlling prediction, adapted from {\cite[Def.~1.1]{angelopoulos2025:learnthentest}}]\label{def:rcp}
    Given an FNO $\operator\learned: \pdeInput \times \hyparamSet \to \pdeOutput$, we postprocess the model output via $\mathcal{T}_{\delta}: \pdeOutput \times \Delta \to \pdeOutput'$, where $\delta \in \Delta$ is a random variable and $\pdeOutput'$ is an arbitrary space. For such a function $\mathcal{T}_{\delta}$, we define a risk $R(\mathcal{T}_{\delta}) \in \real$ that captures a problem-specific notion of the statistical error rate. We say $\mathcal{T}_{\delta}$ is an $(\failrate,\confidence)$-risk-controlling prediction (RCP) if $\mathbb{P}(R(\mathcal{T}_{\delta}) \le \failrate) \ge 1-\confidence$ over the draw of the calibration set from $\mathcal{C}$.
\end{definition}

To instantiate \cref{def:rcp} in our problem with $(\failrate\calib, \confidence\calib)$, we explicitly define $\mathcal{T}_{\delta}\doteq\reach\learned^{\delta}(0)=\{\state\mid V\learned(\state,0)\leq\delta\}$ and $R(\mathcal{T_{\delta}})\doteq\mathbb{P}_{\state\in\mathcal{T}_{\delta}}(\hat J(\state,0,\ctrl\nominal)\geq 0)$. Hence, if $\delta$ is chosen such that $\mathcal{T}_{\delta}$ is ($\failrate\calib,\confidence\calib$)-RCP, then we recover an analogous certification to \eqref{eq:scenario_conclusion} over the calibration distribution $\mathcal{C}$. To find such $\delta$, LTT searches a fixed discrete set $\Delta = \{\delta_1,\dots,\delta_P\}$ with $\delta_i \le 0$ and $|\Delta| = P$ and tests each $\delta_i$ against a null hypothesis $\mathbb{H}_i$. Each $\mathbb{H}_i$ is a statement about a deterministic functional of the trained model $\operator\learned$ and the distribution $\mathcal{C}$, and carries no dependence on the calibration sample. 
Thus, $\mathbb{H}_i : R(\mathcal{T}_{\delta_i}) > \failrate\calib$ is well posed. A $p$-value $p_i$ is a statistic of the calibration draw such that, whenever $\mathbb{H}_i$ holds, $\mathbb{P}(p_i \le \alpha) \le \alpha$ for all $\alpha \in [0,1]$.
Testing over $P$ hypotheses \textit{simultaneously} requires an algorithm that controls the family-wise error rate (FWER) at $\confidence\calib$; we adopt the Bonferroni correction, which tests each hypothesis at $\confidence\calib/P$ so that a union bound caps the FWER at $\confidence\calib$.

\begin{proposition}[Bonferroni controls the FWER, adapted from {\cite[\S9.1]{lehmann2022:testing}}]
\label{prop:bonferroni}
Consider $P$ null hypotheses $\mathbb{H}_1,\dots,\mathbb{H}_P$ with valid $p$-values $p_1,\dots,p_P$ (i.e.,\ $\mathbb{P}(p_i \le \alpha) \le \alpha$ for all $\alpha \in [0,1]$ whenever $\mathbb{H}_i$ is true), and let $\mathcal{I}_0$ denote the unknown set of indices for which $\mathbb{H}_i$ is true. The \emph{Bonferroni correction} at level $\confidence$ returns a set $\mathcal{I}_{\rm F} = \{i : p_i \le \confidence/P\}$ of indices which $\mathbb{H}_i$ declared false. Under arbitrary dependence among the $p_i$, the probability that 
$\mathcal{I}_{\rm F} \cap \mathcal{I}_0 \neq \emptyset$ is at most $\confidence$:
\begin{equation}\tiny\nonumber
\mathrm{FWER} \doteq
\mathbb{P}\Big(\mathcal{I}_{\rm F} \cap \mathcal{I}_0 \neq \emptyset\Big) \le \sum_{i \in \mathcal{I}_0} \mathbb{P}\big(p_i \le \confidence/P\big) \le |\mathcal{I}_0|\frac{\confidence}{P} \le \confidence,
\end{equation}
where the first inequality is a union bound and the second uses validity of each $p_i$, $i \in \mathcal{I}_0$.
\end{proposition}

\noindent\textbf{Procedure:} To exemplify \cref{prop:bonferroni}, we describe the LTT steps. 

\begin{enumerate}[label=\arabic*., leftmargin=*]
    \item For each $\delta_i$ in $\Delta = \{\delta_1,\dots,\delta_P\}$, we test the hypothesis $\mathbb{H}_i: R(\mathcal{T}_{\delta_i}) > \failrate\calib$, i.e., \textit{rollouts are unsafe with chances greater than $\failrate\calib$}. Thus, we seek to declare $\mathbb{H}_i$ false.

    \item Draw $M$ i.i.d.\ pairs $\{\unsafe_j,\state_j^{(i)}\}_{j=1}^M \sim \mathcal{C}$, where $\unsafe_j\in \mathrm{G}$ is an obstacle configuration,
    and $\state_j^{(i)} \sim \mathrm{Unif}(\reach_{\phi,j}^{\delta_i}(0))$ with $V_{\phi,j}$ queried at $\unsafe_j$. We then roll out from $\state_j^{(i)}$ under $\ctrl\induced(\cdot)$ to obtain $\hat J_j^{(i)}(\cdot)$. Note that calibration does not require solving the HJI-VI for each sample:
    since $\hat J(\cdot) \ge V\trueVal(\cdot)$ by construction, certifying $\hat J$ certifies $V\trueVal$ as well. 

    \item Count the failures $K_i = |\{j \le M : \hat J_j^{(i)} \ge 0\}|$. Let $\eta_i \doteq \mathbb{P}(\hat J \ge 0) = R(\mathcal{T}_{\delta_i})$, then $K_i \sim \mathrm{Bin}(M, \eta_i)$ and $\mathbb{P}\big(\mathrm{Bin}(M,\eta_i) \le K_i\big)$ increases as $\eta_i$ decreases. If $\mathbb{H}_i$ is true then $\eta_i > \failrate\calib$. Thus, we take $p_i$ as the largest probability over all $\eta_i$ that $\mathbb{H}_i$ admits, so when $\mathbb{H}_i$ is declared false, $p_i$ remains valid for all of them:
    \begin{equation}\label{eq:p-value}\small
        p_i \doteq \sup_{\eta_i > \failrate\calib}
        \mathbb{P}\big(\mathrm{Bin}(M,\eta_i) \le K_i\big) = \mathbb{P}\big(\mathrm{Bin}(M,\failrate\calib) \le K_i\big),
    \end{equation}
    where the equality follows from continuity and monotonicity of binomial CDF in $\eta_i$, and $p_i$ is a valid $p$-value.

    \item Declare $\mathbb{H}_i$ false if: $p_i \le \confidence\calib/P$. Then $\delta_i$-sublevel set is certified with $\mathbb{P}(\hat J\ge0)\le\failrate\calib$, which holds over the samples from $\mathcal{C}$ with confidence at least $1 -\confidence\calib/P$.

    \item Repeat for every $\delta_i \in \Delta$. 
    Redraw i.i.d. $\{\unsafe_j,\state_j^{(i)}\}_{j=0}^M$ for each one
    and collect $\mathcal{I}_{\rm F} = \{i : p_i \le \confidence\calib/P\}$.

    \item Return the \textit{shallowest} certified level $\delta\floor = \max\{\delta_i : i \in \mathcal{I}_{\rm F}\}$. If $\mathcal{I}_{\rm F} = \emptyset$, this indicates that FNO needs further training or that $\Delta$ must extend deeper.
\end{enumerate}
Therefore, by \cref{def:rcp} and \cref{prop:bonferroni}, with confidence at least $1-\confidence\calib$ over the calibration samples from $\mathcal{C}$, every level $\delta_i$, $i \in \mathcal{I}_{\rm F}$, is certified by $R(\mathcal{T}_{\delta_i}) \le \failrate\calib$: a state inside the $\delta_i$-sublevel set satisfies 
claim \ref{itm:inside-true-set} and \ref{itm:safe-reach-avoid},
i.e., $\hat J(\state,0,\ctrl\learned) < 0$ with probability at least $1 - \failrate\calib$. Each $\mathcal{T}_{\delta_i}$ is ($\failrate\calib,\confidence\calib$)-RCP.

\paragraph{Online Volume Recovery}

    Since $\delta\floor$ is determined across sampled obstacle configurations and initial states, its sublevel set is typically more conservative than what a particular queried configuration admits.
    For the queried reach-avoid set, the online certification recovers the largest safe volume with $\robustDelta>\delta\floor$ by re-certifying against freshly drawn rollouts.
    Thus, $\delta\floor$ only acts as fallback, which is certified independently of the online draws. 
    \cref{tab:certified-planning} shows that $\delta\floor$ is rarely exercised on an unseen obstacle map.

    \textbf{Procedure:} Given $(\failrate\desired, \confidence\desired)$, we now establish that hypothesis testing admit an equivalent formulation as robust scenario optimization in
    \cref{theorem:robustSceneOpt}.
    We pre-register a descending grid $\Delta' = \{\robustDelta_1,\dots,\robustDelta_{P'}\}$ with $|\Delta'|=P'$ and $\robustDelta_{P'} > \delta\floor$.
    At each level $\robustDelta_i\in\Delta'$, let $N_i$ state be i.i.d. sampled uniformly from $\reach^{\robustDelta_i}\learned(0)$ and $k_i\le N_i$ be number of unsafe rollouts with $\hat J \ge0$.
    We test $\robustDelta_i$ against the hypothesis $\mathbb{H}_i: \mathbb{P}(\hat J \geq 0) > \failrate\desired$ at the corrected level $\confidence\desired/P'$ given by \cref{prop:bonferroni}. As shown in \eqref{eq:p-value}, the valid $p$-value for $\mathbb{H}_i$ is $\mathbb{P}(\mathrm{Bin}(N_i,\failrate\desired) \le k_i)$. If $\mathbb{H}_i$ is declared false, 
    then $p_i \leq \confidence\desired/P'$, and thus we recover exactly the condition \eqref{eq:scenario_bound} with $\confidence\doteq\confidence\desired/P'$. 
   Testing each $\robustDelta_i \in \Delta'$ is thus equivalent to applying \cref{theorem:robustSceneOpt} with $(\failrate,\confidence) = (\failrate\desired, \confidence\desired/P')$.
    This extends robust scenario optimization \cite{Lin2024:robustSceneOpt} from a single pre-specified level to a \textit{simultaneous} certification over $P'$ levels.
    By \cref{prop:bonferroni}, with confidence at least $1-\confidence\desired$, all $\robustDelta_i\in \Delta'$ with $i \in \mathcal{I}_{\rm F}$ are certified sublevel sets, and the \textit{largest} certified volume is recovered at $\robustDelta^\star = \max\{\robustDelta_i : i \in \mathcal{I}_{\rm F}\}$.
    
    \cref{alg:online-cert} summarizes an iterative online certification which is equivalent to $P'$ hypothesis testings. To expedite the online search, we exploit the nesting of the sublevel sets: since
    $\robustDelta_1 > \cdots > \robustDelta_{P'}$, we have
    $\{V\learned(\cdot)\le\robustDelta_{P'}\} \subseteq \cdots \subseteq
    \{V\learned(\cdot)\le\robustDelta_1\}$. Consequently, the sample of size $N$ drawn for $\robustDelta_1$ also fall within a smaller $\robustDelta_i$-sublevel set and remain i.i.d.
    (i.e., $n_i \leq N$ in Line 8). 
    Note that \cref{prop:bonferroni} holds for any dependence among $p_i$.
    Thus, a single round of sampling and trajectory evaluation via \eqref{eq:cost-upper-bound} suffices to evaluate \cref{theorem:robustSceneOpt} at every $\robustDelta_i\in\Delta'$.

    \textbf{Effective Sample Size:} To tighten the certification, we derive a condition on an effective sample size $N$ for performing scenario optimization with
    $(\failrate, \confidence):=(\failrate\desired, \confidence\desired/P)$. Following
    \cite[\S4.3]{campi2011:samplingDiscarding}, a Chernoff bound on the binomial
    tail of \eqref{eq:scenario_bound}
    gives $k/N < \failrate - \sqrt{2\failrate N \ln(1/\confidence)}/N$. The inequality reveals the mismatch between empirical rate $k/N$ and the true $\failrate$ by at least
    $\sqrt{2\failrate N \ln(1/\confidence)}/N$. We bound this gap by a design parameter $\gamma \in (0,1)$ s.t. $\failrate - k/N = \gamma\failrate$. Substituting into the inequality yields
    \begin{equation}\nonumber
      N_{\rm eff} = \left\lceil 2\ln(1/\beta) / (\gamma^{2}\failrate) \right\rceil
    \end{equation}
    so the gap vanishes as $\gamma \to 0$ at a sampling cost of $N_{\rm eff}$.
    
\begin{algorithm}[t]
\caption{Online Certification of FNO-predicted BRAT}
\label{alg:online-cert}
\footnotesize
\begin{algorithmic}[1]
\setlength{\textfloatsep}{8pt}
\Input
$\Delta', \delta\floor, \failrate\desired, \confidence\desired, N_{eff}$
\Output $\delta^\star$
\State $\alpha\gets\confidence\desired/P'$; where $P'=|\Delta'|$
\State $\reach_1\gets\{\state\mid V\learned(\state,0)<\delta_1\}$; where $\delta_1=\Delta'[0]$
\State draw $\state_0,\dots,\state_{N-1}\stackrel{\text{iid}}{\sim}\mathrm{Unif}(\reach_1)$ 
\For{$j=0,\dots,N-1$} \Comment{\textit{N i.i.d. rollouts}}
  \State $v_j\gets V\learned(\state_j)$;\quad
         $f_j\gets\mathds{1}\{\hat J(\state_j, 0, \ctrl_{\phi,j})\ge0\}$
\EndFor
\For{$i=1,\dots,P'$} $\delta_i \gets \Delta'[i-1]$ \Comment{\textit{$P'$ hypothesis testings}}
  \State $n_i\gets\lvert\{j:v_j<\delta_i\}\rvert$;\quad
         $k_i\gets\lvert\{j:v_j<\delta_i,\;f_j\!=\!1\}\rvert$
    \If {\eqref{eq:scenario_bound} is satisfied with $n_i, k_i, \failrate\desired, \alpha$}
    \State \Return $\delta^\star=\delta_i$ \Comment{\textit{the largest is the first to declare false}}
    \EndIf
\EndFor
\State \Return $\delta^\star=\delta\floor$
\end{algorithmic}
\end{algorithm}


\subsection{Contingency-Aware Planning with Certified Reachability}\label{sec:contingency-aware-plan}
In this section, we show how certified FNO sets can be integrated into any global planner in unknown environments by encoding them as a feasible planning region
$\Omega_{\rm feasible}\subseteq\Omega$.
We first establish a translation invariance property of BRAT in \cref{cor:localshiftValue}. This property allows the FNO model $\operator\learned$, which is trained over the local spatial domain $D\local$, to be deployed in the global environment $\Omega$ for planning.

\begin{corollary}\label{cor:localshiftValue}
Suppose that the dynamics $f$ is invariant under planar translation, i.e.\ $f(\pose+c,\theta,\ctrl,\disturb)=f(\pose,\theta,\ctrl,\disturb)$ $\forall c\in\real^2$, with position-independent $\ctrlset$ and $\disturbset$. 
Let $V\learned$ be the value function on $D\local$ for the safe set $\safe\local$ and the unsafe set $\unsafe\local$. 
If $\safe$ and $\unsafe$ are obtained by translating $\safe\local$ and $\unsafe\local$ by $\pose_{\ell}$, 
then $V\glob\learned(\pose,\theta,t)=V\learned(\pose-\pose_{\ell},\theta,t)$ $\forall t\in[0,T]$. 
Since the coordinate shift preserves the dynamics and trajectories under the same $\ctrl$ and $\disturb$, the guarantee of \cref{theorem:robustSceneOpt} is also preserved in the global frame.
\end{corollary}

\subsubsection{Reachability Constraint for Global Planner}\label{sec:rrtx-reach-constraint}
Let $V\glob_{\hyparam}(x,y,\theta,0)$ be certified at $\robustDelta^\star$. Define the
heading-agnostic value function as pointwise maximum over $\theta$ as
\begin{equation}\nonumber\small
  \overline V\glob_{\hyparam}(x,y) := \max_{\theta\in\Theta} V\glob_{\hyparam}(x,y,\theta,0)
\end{equation}
and the 2D certified region
$\Omega_{{\rm reach}} := \{(x,y)\in\Omega \mid \overline V\glob_{\hyparam}(x,y)\le\robustDelta^\star\}$.
Since $\overline V\glob_{\hyparam}\le\robustDelta^\star$ implies
$V\glob_{\hyparam}\le\robustDelta^\star$ for every $\theta$,
the certified inclusion holds:  $\Omega_{{\rm reach}}\times\Theta \subseteq \reach^{\robustDelta^\star}_{\hyparam}(0)$ and
  $\mathbb{P}_{\state\in\reach^{\robustDelta^\star}\learned(0)}\!\left( \state \in \reach\trueVal(0)\right)
    \geq 1-\failrate\desired$. 
For an environment $\Omega$ containing $L$ safe sets, we define the feasible
planning region as
$\Omega_{\rm feasible} := \bigcup_{i=1}^{L}\Omega_{{\rm reach},i}$,
where each $\Omega_{{\rm reach},i}$ is an orientation-agnostic planning region associated with a probabilistically certified reach-avoid set (see \cref{fig:reach-constraint}).

\begin{figure}[b!]
    \centering
    \includegraphics[width=0.9\linewidth]{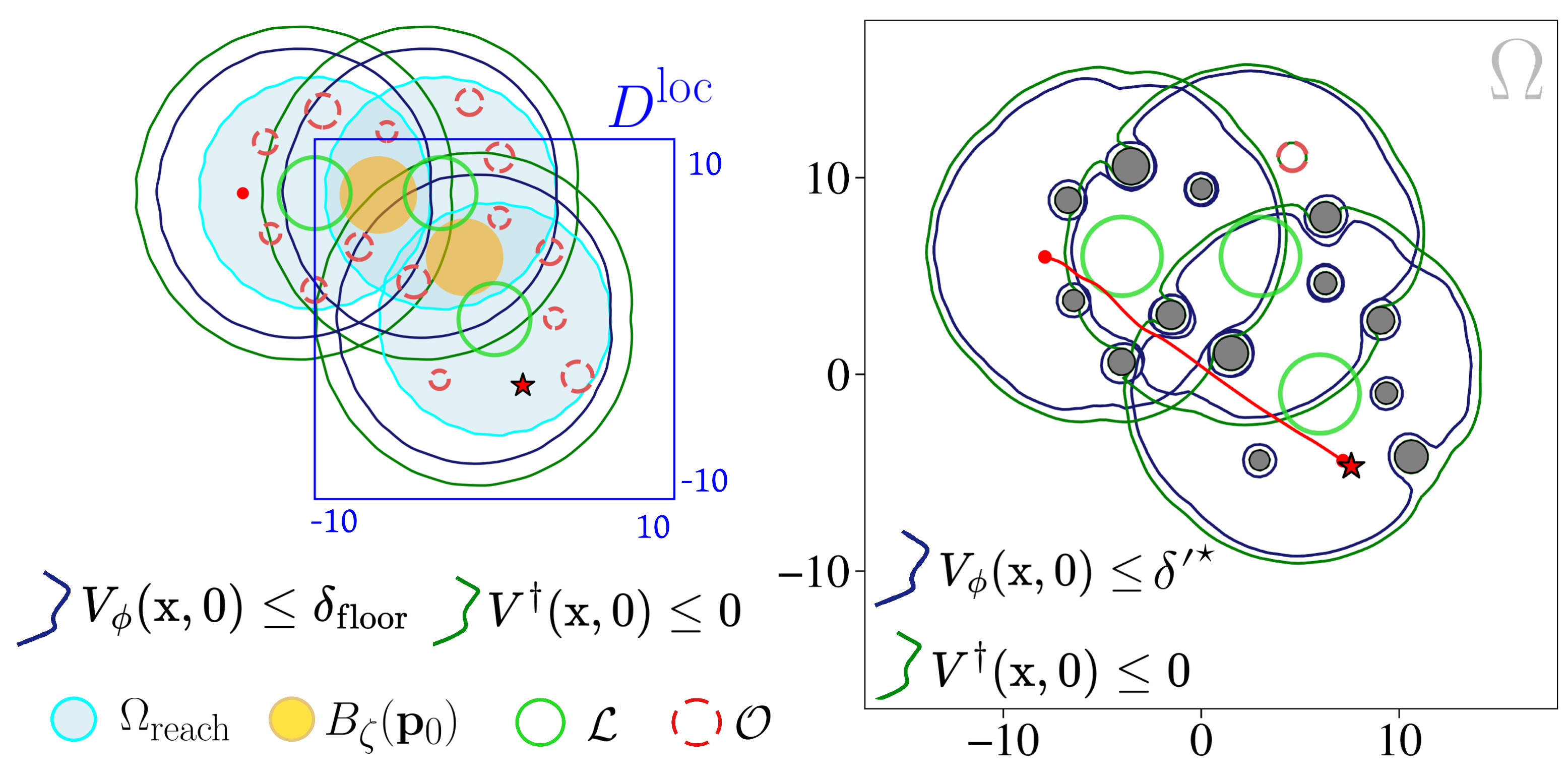}
    \caption{\textbf{Planning with certified FNO reachability (\cref{sec:contingency-aware-plan}).} An FNO trained on a local domain $D\local$ is deployed across $\Omega$ to reach the goal (red star) around initially unknown obstacles (red, dashed). Its certified $\robustDelta^\star$-sublevel set (dark blue, sublevel set) lies inside the true reachable set (green, sublevel set). Feasible planning regions initialized at $\delta\floor$ (cyan) satisfies \cref{assum:overlap} (yellow ball), so the global planner can always advance along the trajectory (red).}
    \label{fig:reach-constraint}
            
\end{figure}

\subsubsection{Contingency Policy}\label{sec:contingency-policy}
Let $\mathcal{E}_{\rm adv}$ denote an unexpected adverse event (e.g., low
battery, sensor failure, or disturbance) occurring at an unknown
time $t_{\rm adv}$. The robot tracks
the global plan to $\pose_{\rm goal}$ using the controller $\ctrl_{\rm global}$ until $\mathcal{E}_{\rm adv}$ occurs, then switches to the recovery policy $\ctrl_{\hyparam,i}^{\mathrm{n}}(\cdot)$ induced by $V\glob_{\hyparam,i}$ via
\eqref{eq:hjr-opt-ctrl} for any $i$ with $\pose(t_{\rm adv})\in\Omega_{{\rm reach},i}$,
which exists since the robot is constrained to $\Omega_{\rm feasible}$. By
\cref{theorem:robustSceneOpt}, the open-loop policy $\ctrl_{\hyparam,i}^{\mathrm{n}}(\cdot)$ along the \textit{disturbance-free} rollout from $\state(t_{\rm adv})$ reaches $\safe_i$ in \textit{finite-time} under any admissible disturbance $\disturb\in\disturbset$, independently of the global planner.

\subsubsection{Design Considerations for Feasible Planning}\label{sec:design-constraint-for-contingency}
We state two requirements on the planning and sensing parameters
under which the certified sets remain valid during execution.

\begin{assum}[Sufficient Overlap of Certified Regions]\label{assum:overlap}
Prior to planning, let $\Omega_{{\rm reach},i}$ and $\Omega_{{\rm reach},j}$ be the certified $\delta\floor$-sublevel sets of
$\overline{V}\glob_{\phi,i}$ and $\overline{V}\glob_{\phi,j}$ associated
with the safe sets $\safe\glob_i$ and $\safe\glob_j$, and let $\rho > 0$
be the maximum step length of the global planner. We assume that consecutive certified regions along the mission admits a
$\pose_0$ such that
\begin{equation}\nonumber\small
  B_\zeta(\pose_0) \subseteq
  \Omega_{{\rm reach},i} \cap \Omega_{{\rm reach},j},
  \qquad \zeta := r_{\obs_{\max}} + |\delta\floor| + \rho,
\end{equation}
where $B_\zeta(\pose_0)$ is the closed ball of radius $\zeta$ about $\pose_0$.
\end{assum}
Since HJI-VI~\eqref{eq:hji-vi} enforces $V(\cdot,t) \ge g \; \forall t$, an obstacle of radius $r_\obs$ excludes a disk of radius $r_\obs + |\delta\floor|$ from $\Omega_{\rm reach}$. \Cref{assum:overlap} thus keeps the union of certified regions connected, so the planner can step between $\Omega_{{\rm reach},i}$ and $\Omega_{{\rm reach},j}$ even with the largest obstacle present (\cref{fig:reach-constraint}). Because each $\safe\glob_i$ is fixed, the assumption can be met by training FNO with longer horizon $T$ to enlarge each certified region and hence the overlap. If newly discovered obstacles leave no path within the intersection, $\pose_{\rm goal}$ is unreachable, 
and the robot executes the contingency policy of \cref{sec:contingency-policy}.

\textbf{LIDAR sensing range:} $r_{\rm sense}$ must cover the displacement
attainable over the reach horizon $T$. Since $\|\dot\pose\|_2 \le v_{\max} +
\bar d$ for all admissible $\disturb(\cdot)$, integration over $[0,T]$ gives
\begin{equation}\nonumber\small
  \|\pose(T) - \pose(0)\|_2 \le {\textstyle\int_0^T} \|\dot\pose\|_2 \,dt
  \le (v_{\max} + \bar d)\,T \doteq r_T .
\end{equation}
We therefore require $r_{\rm sense} \ge r_T$. A contingency maneuver triggered
at $t_{\rm adv}$ remains in $B_{r_T}(\pose(t_{\rm adv}))$ over
$[t_{\rm adv}, t_{\rm adv}+T]$, so every obstacle capable of intersecting the
maneuver is already observed at $t_{\rm adv}$. Any obstacle revealed thereafter
lies beyond $r_T$ from $\pose(t_{\rm adv})$, so the certificate and the recovery policy $\ctrl_{\phi,i}^{\mathrm{n}}(\cdot)$ remain valid over the horizon.

\begin{figure}[t]
    \centering
    \includegraphics[width=0.9\linewidth]{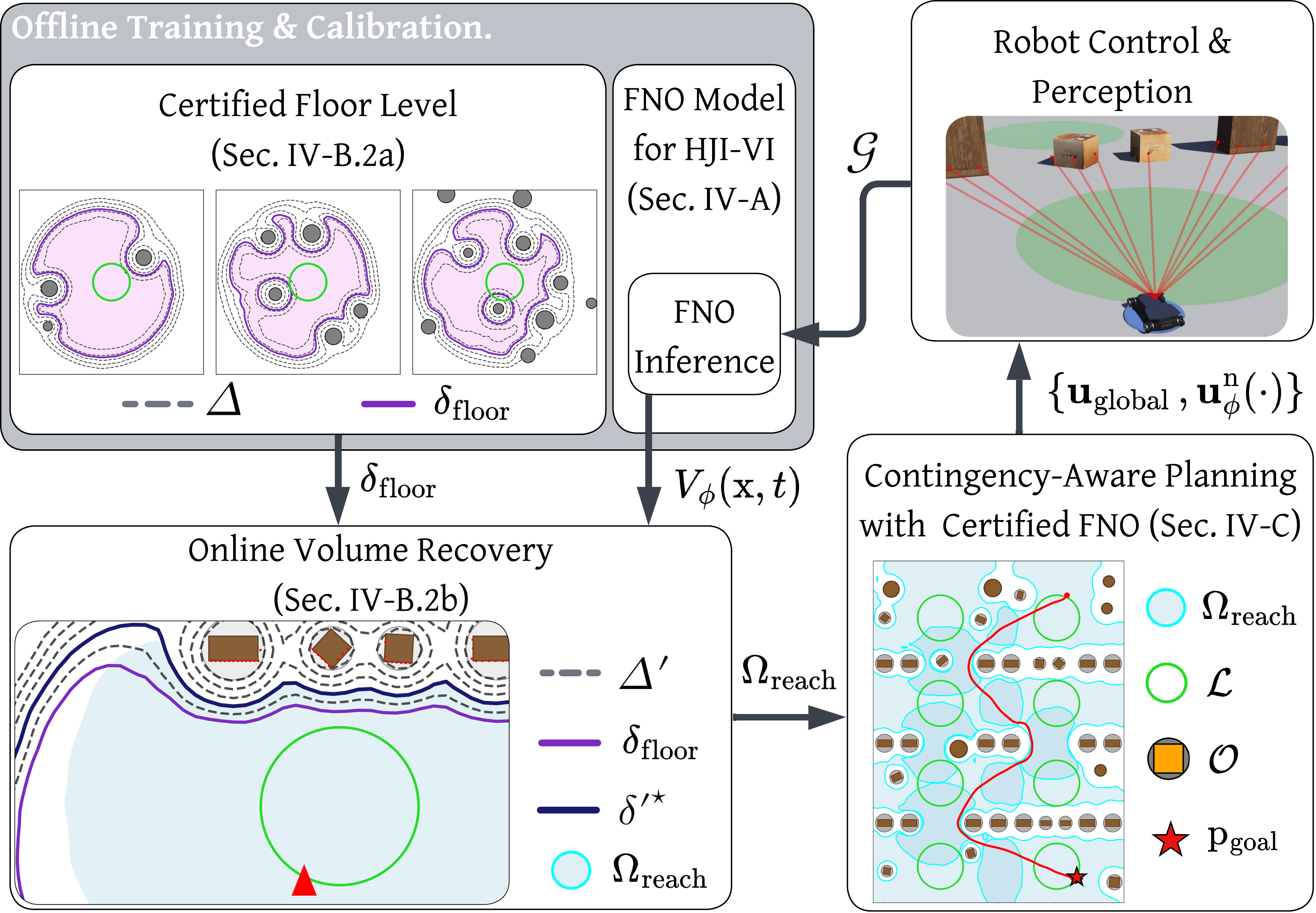}
    \caption{\textbf{Overall framework.} The FNO is trained offline and certified at the floor $\delta\floor$ over a calibration set. On each obstacle detection, FNO predicts the updated reach-avoid sets, which are re-certified at $\robustDelta^\star > \delta\floor$ to recover volume. The recovered sets define the certified region $\Omega_{\rm reach}$ available to the contingency-aware planner.}
    \label{fig:overall-framework}
\end{figure}

\textbf{Overall framework} (\cref{fig:overall-framework})\textbf{:} 
As obstacles are revealed, the FNO rapidly updates the reach-avoid sets, which are re-certified online to recover the largest admissible planning region, with $\delta\floor$ serving as a pre-certified fallback. The global planner is constrained to the resulting feasible region $\Omega_{\rm feasible}$ and uses $\ctrl_{\rm global}$ for mission progress. If an adverse event occurs, the robot switches to the contingency policy $\ctrl\induced$. Under this policy, finite-time recovery to a safe region is certified against any admissible disturbance, at $(\failrate\desired,\confidence\desired)$ when online certification succeeds and $(\failrate\calib,\confidence\calib)$ otherwise.


\section{Simulation Results}\label{sec:results}
\subsection{FNO Training and Implementation}\label{sec:results-fno-training}
We generate $N_g=250$ ground-truth samples for \eqref{eq:unicycleModel} with $v \in [0,1]$, $\omega \in [-1,1]$, and $\disturb_{\max} = [0.02,0.02,0.01]$, over $T=8$ on a $[50,50,25,17]$ grid in $[x,y,\theta,t]$. Each sample has $\safe$ centered at the origin ($r_\ell=2$) in $D\local=[-10,10]^2$ and $h \sim \mathrm{Unif}\{1,\dots,16\}$ obstacles with $(x_{\obs_h}, y_{\obs_h}) \in D\local$ and $r_{\obs_h} \in [0.5,1]$. The FNO has $B=4$ Fourier layers ($d_v=24$),  inputs $\{g(\state), x, y, \theta, t\}$ ($d_a=5$), output $V\learned$ ($d_y=1$), and ReLU activations. Training takes 120 epochs with AdamW (${\approx}2$\,h on an i7-13700H, 16\,GB RAM, RTX 4050).

With $M = 15{,}000$ calibration samples,
$\failrate\calib = 0.005$, and $\confidence\calib = 10^{-6}$, we tested on $P=41$ evenly spaced values over $[-1.9, -0.6]$ and obtain $\delta\floor = -1.055$ for our FNO model, where $|\mathcal{I}_{\rm F}| = 22$ declared false. 
For online certification with $\failrate\desired=0.05$ and $\confidence\desired = 10^{-6}$, \cref{alg:online-cert} tests on $P'=30$ evenly spaced values over $(\delta\floor, 0]$ with $N_{\rm eff}=7,652$ ($\gamma=0.3$). For all planning results, we leverage FNO zero-shot super-resolution to predict a BRAT over grid of size $[64,64,25,17]$.


\subsection{Safety Validation of the Contingency Policy}
We validate the safety guarantee of the induced control policy $\ctrl\learned$ on $4000$ maps with
$h \in \{3,5,7,11\}$ obstacles. We substep the $t$-grid to obtain a $0.1$\,s control period. In each map, a rollout starts at $x_0 \sim \mathrm{Unif}\{\state\mid V\learned\le \delta\floor\}$
under a disturbance $\disturb \sim \mathrm{Unif}\bigl([-\disturb_{\max}, \disturb_{\max}]\bigr)$ redrawn every $0.5$\,s. Since $\hat J\geq J$ for admissible $\disturb$, \cref{theorem:robustSceneOpt} guarantees $\mathbb{P}(J \ge 0) \le \failrate\calib$ with
confidence at least $1-\confidence\calib$. 

Define the Clopper-Pearson upper bound \cite[\S3.8.3]{johnson2005:univariate}:
$\failrate\cp = \sup\bigl\{\eta \in [0,1] : \mathbb{P}(\mathrm{Bin}(N,\eta) \le k) \ge \confidence\calib\bigr\}$, i.e.,
the largest failure rate consistent with observation $(N,k)$. Across all maps, \cref{tab:contingency}
reports $k=2$ failures, all from not reaching the safe set by $T=8$\,s, and no collisions. Hence, $\failrate\cp = 0.48\% < \failrate\calib$, which is consistent with the guarantee of \cref{theorem:robustSceneOpt}. Moreover, for each obstacle count,
$\failrate\cp \le 1.66\% < \failrate\desired$, so $\delta\floor$ remains conservative across
distribution of obstacle maps and serves as a reliable fallback.

\begin{table}[b!]
\centering
\scriptsize
\renewcommand{\arraystretch}{1}
\setlength{\aboverulesep}{0.3ex}
\setlength{\belowrulesep}{0.4ex}
\setlength{\cmidrulesep}{0.2ex}
\setlength{\abovecaptionskip}{5pt}
\setlength{\belowcaptionskip}{5pt}
\caption{Contingency policy validation at $\delta_{\rm floor}=-1.055$ with $\failrate\calib = 0.5\%$.
Each simulated rollout starts at $x_0 \sim \mathrm{Unif}\{\state\mid V\learned\le \delta\floor\}$ under a disturbance
$\disturb_k \overset{\text{i.i.d.}}{\sim} \mathrm{Unif}\bigl([-\disturb_{\max}, \disturb_{\max}]\bigr)$.} 
\label{tab:contingency}
\renewcommand{\arraystretch}{1.1}
\setlength{\tabcolsep}{4.5pt}
\begin{tabular}{@{}lccccc@{}}
\toprule
 & \textbf{3 Obs} & \textbf{5 Obs} & \textbf{7 Obs} & \textbf{11 Obs} & \textbf{All} \\
\midrule
Maps count                 & 1000  & 1000  & 1000  & 1000  & 4000 \\
Failures $J \ge 0$ $\big/$ collisions $g \ge 0$ & 1 / 0 & 1 / 0 & 0 / 0 & 0 / 0 & 2 / 0 \\
\midrule
$\Pr[J\ge0]$ Clopper-Pearson bound & 1.66\% & 1.66\% & 1.37\% & 1.37\% & \textbf{0.48\%} \\
\bottomrule
\end{tabular}
\end{table}


\subsection{Planner Evaluation with Certified Sets}\label{sec:planner-eval}
We integrate the reachability constraints of \cref{sec:contingency-aware-plan} ($\failrate\desired = 0.05$) into two planners suited for unknown environments. (1) RRT$^X$~\cite{otte2016:rrtx} grows a goal-rooted tree, so every node stores its cost-to-goal. When new obstacles invalidate edges, a local rewiring cascade reassigns each affected node to its lowest-cost neighboring parent.
(2) SCRAMPPI~\cite{srirangam2026:scramppi} is a model predictive path integral controller that samples control sequences around guided paths to the goal, which are generated by a topological probabilistic roadmap.
We evaluate both planners in two environments on $\Omega = [-15,15]^2$\,m with pre-defined safe regions
(see \cref{fig:planning-env}). Each environment uses 10 random seeds to generate $>16$ unknown obstacles positions and radii, detectable within $r_{\rm sense} \ge r_T\approx 8.23$\,m.

\cref{tab:certified-planning} reports metrics averaged over 20 random seeds. Each of the $n$ obstacles discoveries
triggers a reachable-set recomputation taking $\tau_{\rm reach}$: query time plus certification time for FNO, and the exact solve for HJ. The FNO reduces this cost by $\approx2.2\times$ relative to exact HJ 
and shortens total execution time $T_{\rm exec}$ by $\approx32\%$ across all settings. Certified volume recovery leaves path quality intact: the normalized path length $\bar{L}$, defined as the executed trajectory length divided by the optimal path under a fully known map, remains within $2\%$ of the exact HJ baseline in every configuration. All trajectories confined to $\Omega_{\rm feasible}$ satisfy $V\trueVal(\state,0) < 0$, so no executed state lies outside the true reach-avoid set. Finally, online certification at each FNO set rarely falls back to $\delta\floor$. In the disturbance-free setting (retrained as in \cref{sec:results-fno-training} with $\disturb=\mathbf{0}$, $\delta\floor= -0.7875$), all queried sets certify a sublevel set strictly above the floor ($100.0\%$). Under disturbance, the recovery rate drops only to $99.7\%$-$99.8\%$. Since $r_\pose(t)$ in $\hat{J}(\cdot)$ grows large for states far from the safe region, the shallow sublevel sets are the first to lose their certificates. When no level clears $\failrate\desired$, the planner reverts to $\delta\floor$, certified independently of the online draw.
\begin{figure}[t]
    \centering
    \includegraphics[width=0.9\columnwidth]{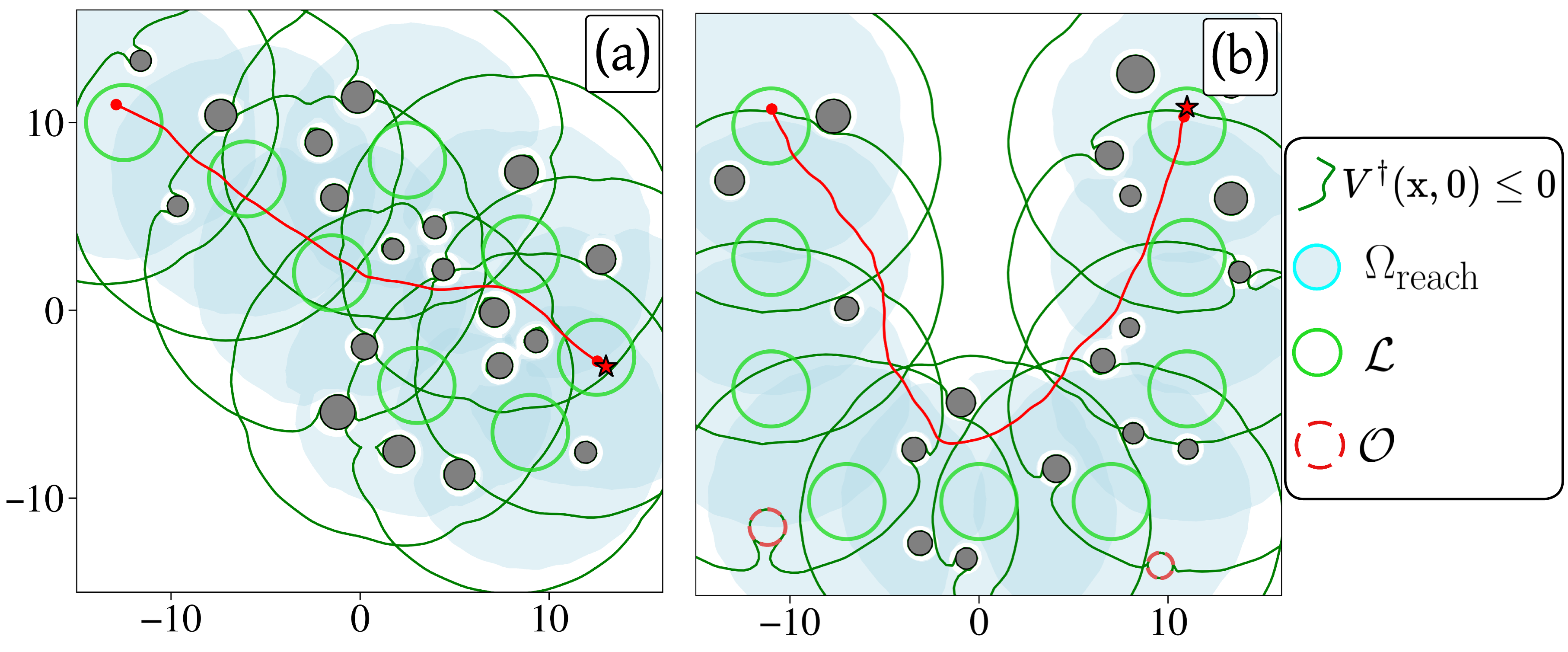}
    \caption{Environments for planners evaluation (\cref{sec:planner-eval}). (a) $L=8$ clustered safe regions admit multiple distinct paths. (b) $L=9$ sequential safe regions force a single mission corridor.}
    \label{fig:planning-env}
\end{figure}

\begin{table}[t]
\centering
\setlength{\tabcolsep}{3pt}
\renewcommand{\arraystretch}{1}
\setlength{\aboverulesep}{0.3ex}
\setlength{\belowrulesep}{0.4ex}
\setlength{\cmidrulesep}{0.2ex}
\setlength{\abovecaptionskip}{5pt}
\setlength{\belowcaptionskip}{5pt}
\scriptsize
\caption{Planner performance with certified FNO sets ($\failrate\desired=0.05$) versus exact HJ sets, averaged over 20 maps per planner. $n$: number of obstacles reveals; $\tau_{\mathrm{reach}}$: BRAT recomputation time per reveal; $T_{\mathrm{exec}}$: total execution time; $\bar{L}$: path length normalized by the fully-known-map optimum; $V\trueVal(\state,0){<}0$: fraction of states inside the true BRAT; $\robustDelta^\star{>}\delta\floor$: fraction of online certification without fallback (pooled counts over total of per-set certifications).}
\label{tab:certified-planning}
\begin{tabular}{lrrrrr}
\toprule
Planner & ${n}{\times}\tau_{\mathrm{reach}}$\,[ms] & $T_{\mathrm{exec}}$\,[s] & $\bar{L}$ & $V\trueVal<0$ & $\robustDelta^\star{>}\delta\floor$ \\
\midrule
\multicolumn{6}{l}{\emph{Disturbed $(\delta\floor= -1.055)$}} \\
\midrule
\cite{otte2016:rrtx} + FNO   & 15.7${\times}$738  & \textbf{29.1} & 1.072 & 100\% & 99.7\% (1,185/1,189)\\
\cite{otte2016:rrtx} + HJR   & 15.6${\times}$1633 & 41.8 & 1.050 &  100\% & -- \\
\cite{srirangam2026:scramppi} + FNO           & 16.3${\times}$735  & \textbf{34.5} & 1.042 &  100\% & 99.8\% (1,176/1,178) \\
\cite{srirangam2026:scramppi} + HJR& 16.1${\times}$1617 & 52.7 & 1.045 &  100\% & -- \\
\midrule
\multicolumn{6}{l}{\emph{Disturbance-free $(\delta\floor= -0.7875$)}} \\
\midrule
\cite{otte2016:rrtx} + FNO   & 14.9${\times}$758  & \textbf{26.7} & 1.030 &  100\% &  100\% (1,152/1,152)\\
\cite{otte2016:rrtx} + HJR   & 15.2${\times}$1512 & 38.4 & 1.040 &  100\% & -- \\
\cite{srirangam2026:scramppi} + FNO           & 15.3${\times}$736  & \textbf{33.6} & 1.014 &  100\% &  100\% (1,121/1,121) \\
\cite{srirangam2026:scramppi} + HJR& 15.9${\times}$1607 & 47.1 & 1.041 &  100\% & -- \\
\bottomrule
\end{tabular}
\end{table}


\subsection{Empirical Support for Choice of FNO}
\subsubsection{FNO vs.\ $\mu$-DeepReach}\label{sec:fno-vs-deepreach}
As discussed in \cref{sec:FNO-choice-support}, the FNO generalizes across obstacle maps through its function-valued input, whereas $\mu$-DeepReach must fix the number of obstacles in $\mu$. We nonetheless compare both on three-obstacle maps with $\mu = [x_\obs^i, y_\obs^i, r_\obs^i]_{i=1}^3$ and the kinematics of \cref{sec:results-fno-training}. $\mu$-DeepReach is a three-layer SIREN with 512 neurons per layer, trained self-supervised on the HJI-VI residual $\|\max\{\min\{\partial_t V\learned + \ham, \ell - V\learned\}, g(\mu) - V\learned\}\|$ and the terminal loss $\|V\learned - \max\{\ell, g(\mu)\}\|_{t=T}$ for 10K pretraining and 90K curriculum iterations ($\approx 7$\,h, final loss $\approx 0.017$). Being supervised directly on exact HJ solutions, the FNO achieves lower relative $L_2$ error and $\mathrm{Vol}_{\mathrm{FN}}$ and higher IoU at every resolution (\cref{tab:fno-vs-deepreach}), despite training on data with up to 16 obstacles and zero-shot evaluation above its training resolution. $\mu$-DeepReach under-approximates the true reach--avoid set (\cref{fig:fno-vs-deepreach}), which lowers $\mathrm{Vol}_{\mathrm{FP}}$ at the cost of IoU, though $\mathrm{Vol}_{\mathrm{FP}}$ is small for both. The FNO is also $8$--$10\times$ faster to query, as it predicts the entire grid in one forward pass rather than evaluating each point independently.

\begin{table}[t]
\centering
\scriptsize
\setlength{\tabcolsep}{3pt}
\renewcommand{\arraystretch}{1}
\setlength{\aboverulesep}{0.2ex}
\setlength{\belowrulesep}{0.3ex}
\setlength{\abovecaptionskip}{5pt}
\setlength{\belowcaptionskip}{0pt}
\caption{FNO vs.\ $\mu$-DeepReach on 30 random three-obstacles maps at grid $[\chi,\chi,25,17]$. Relative $L_2$ error is computed point-wise against the true BRAT. 
At $0$-sublevel set, $\mathrm{Vol}_{\mathrm{FP}} = |\reach\learned \setminus \reach\trueVal| / |\reach\trueVal|$, $\mathrm{Vol}_{\mathrm{FN}} = |\reach\trueVal \setminus \reach\learned| / |\reach\trueVal|$, and
 $\mathrm{IoU}= |\reach\learned\cap\reach\trueVal|/|\reach\learned\cup\reach\trueVal|$.}\label{tab:fno-vs-deepreach}
\begin{tabular}{@{}llrrrrr@{}}
\toprule
$\chi$ & Model & query [s] & Rel.\ $L_2$ & $\mathrm{Vol}_{\mathrm{FP}}$ & $\mathrm{Vol}_{\mathrm{FN}}$ & IoU \\
\midrule
\multirow{2}{*}{64}
 & FNO             & \textbf{0.18}  & \textbf{0.035} & 0.0019          & \textbf{0.0106} & \textbf{0.988} \\
 & $\mu$-DeepReach & 1.40           & 0.322          & \textbf{0.0009} & 0.1511          & 0.848          \\
\midrule
\multirow{2}{*}{128}
 & FNO             & \textbf{0.59}  & \textbf{0.075} & 0.0012          & \textbf{0.0276} & \textbf{0.971} \\
 & $\mu$-DeepReach & 5.50           & 0.339          & \textbf{0.0002} & 0.1558          & 0.844          \\
\midrule
\multirow{2}{*}{256}
 & FNO             & \textbf{2.24}  & \textbf{0.102} & 0.0013          & \textbf{0.0394} & \textbf{0.959} \\
 & $\mu$-DeepReach & 21.86          & 0.351          & \textbf{0.0001} & 0.1614          & 0.839          \\
\bottomrule
\end{tabular}
\end{table}

\begin{figure}[t]
  \centering
  \includegraphics[width=1\columnwidth]{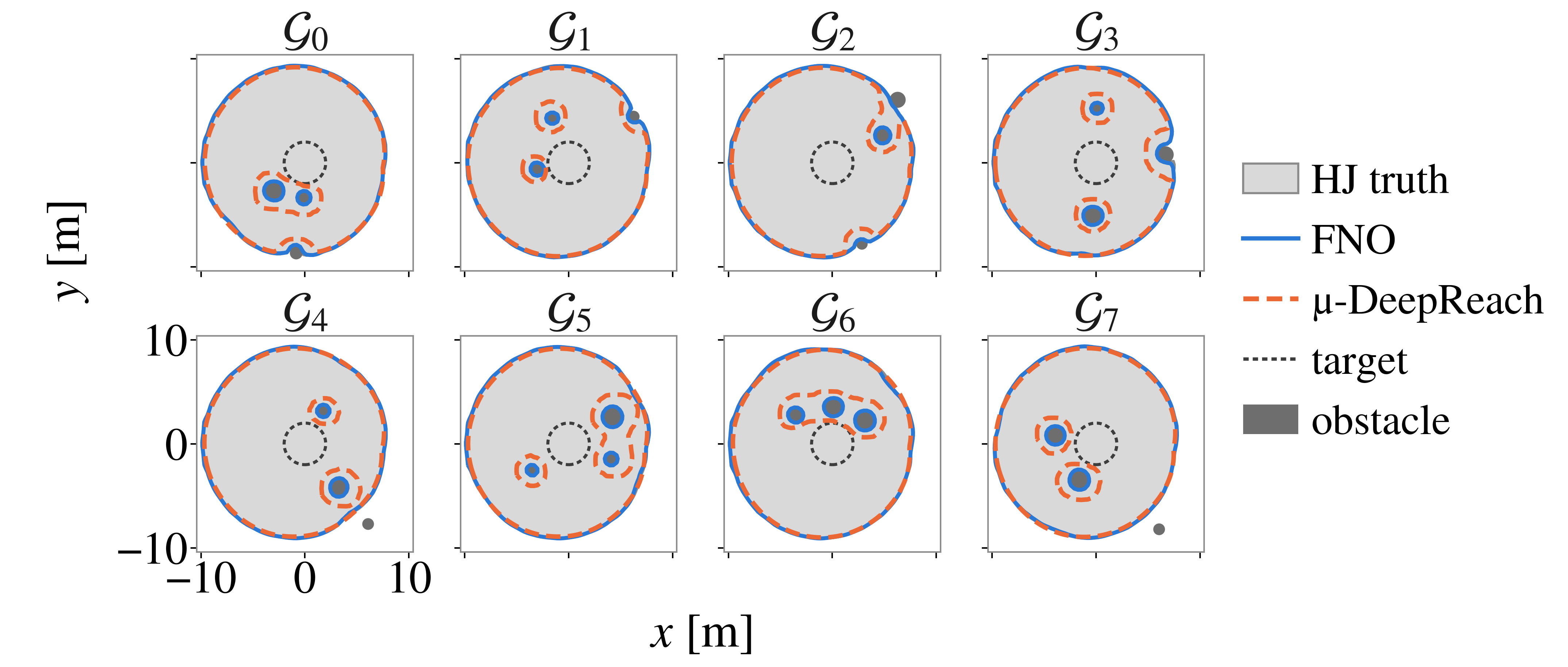}
  \caption{Zero-level sets of the FNO (blue) and $\mu$-DeepReach (orange) compared with the true BRAT (grey, filled) on three-obstacles maps.}
  \label{fig:fno-vs-deepreach}
\end{figure}

\subsubsection{FNO vs.\ HJ Solver Computation Time}\label{sec:fno-vs-hj}
We compare the computation time of the FNO and the exact HJ solver \texttt{OptimizedDP} across spatial discretizations in \cref{tab:reach-time-sweep}. The FNO time includes both the model query and the scenario-based certification of \cref{sec:set-certification}. The FNO achieves up to $8.4\times$ speed-up over the numerical solver and remains viable for real-time use ($<1$\,s) up to a spatial discretization of $\chi = 128$. The certified learned set at $\robustDelta^\star$ recovers up to $78\%$ of the true BRAT, and every certified state lies inside the true BRAT which is a result of \cref{theorem:robustSceneOpt}.

\begin{table}[htb!]
\centering
\scriptsize
\setlength{\tabcolsep}{3.5pt}
\renewcommand{\arraystretch}{1}
\setlength{\aboverulesep}{0.3ex}
\setlength{\belowrulesep}{0.4ex}
\setlength{\cmidrulesep}{0.2ex}
\setlength{\abovecaptionskip}{5pt}
\setlength{\belowcaptionskip}{5pt}
\caption{FNO vs.\ \texttt{OptimizedDP} computation time on 20 random maps with 1--16 obstacles at grid $[\chi,\chi,25,17]$. FNO time is the sum of operator query and its certification. Speed-up is the HJR time over the FNO total. For the certified set at $\robustDelta^\star$, true safe volume is $|\reach\learned\cap\reach\trueVal|/|\reach\learned|$ and recovered volume is $|\reach\learned\cap\reach\trueVal|/|\reach\trueVal|$.}
\label{tab:reach-time-sweep}
\begin{tabular}{@{}crrrcc@{}}
\toprule
& \multicolumn{3}{c}{Mean compute time [s]} & \multicolumn{2}{c}{Certified Volume at $\robustDelta^\star$} \\
\cmidrule(lr){2-4}\cmidrule(lr){5-6}
$\chi$ & \multicolumn{1}{c}{FNO (query\,+\,cert.)} & \multicolumn{1}{c}{HJR} & Speed-up & True safe & Recovered \\
\midrule
 64 & 0.334 (0.145\,+\,0.190) &  0.974 & 2.91$\times$ & 1.000 & 0.779 \\
128 & 0.859 (0.557\,+\,0.303) &  4.067 & 4.73$\times$ & 1.000 & 0.777 \\
256 & 3.101 (2.234\,+\,0.867) & 26.040 & 8.40$\times$ & 1.000 & 0.771 \\
\bottomrule
\end{tabular}
\end{table}
\section{Conclusion}
We developed a learning-based Hamilton--Jacobi (HJ) reach-avoid framework that combines a Fourier Neural Operator with two-stage scenario-based certification for real-time contingency-aware planning. An offline stage certifies a conservative sublevel set over the calibration map distribution; an online stage recovers a larger one for each queried map. This extension of scenario-based certification from a single set to a family of sublevel sets
admits efficient online certification at a fraction of the cost of exact HJ solvers.


{\def\footnotesize{\fontsize{8pt}{9pt}\selectfont}
\bibliographystyle{IEEEtran}
\bibliography{main_truncated}}
\end{document}